\documentclass[letterpaper, 10 pt, conference]{ieeeconf}  

\IEEEoverridecommandlockouts                              

\usepackage{times}
\usepackage{amsmath}

\usepackage{amsthm}
\usepackage{amsfonts}
\usepackage{color}
\usepackage{array}
\usepackage{graphicx}
\graphicspath{ {./figs/} }
\usepackage{wrapfig}
\usepackage{mathtools}
\usepackage{amssymb}
\usepackage{floatrow}
\usepackage[ruled,noend,linesnumbered]{algorithm2e}
\usepackage{wasysym}
\usepackage{todonotes}
\usepackage{tabularx}
\usepackage{enumerate}
\usepackage{multirow}
\usepackage{diagbox}
\usepackage{setspace}
\newcolumntype{b}{X}
\newcolumntype{m}{>{\hsize=.8\hsize}X}
\newcolumntype{s}{>{\hsize=.5\hsize}X}

\newtheorem{theorem}{Theorem}
\newtheorem{problem}{Problem}

\newtheorem{corollary}{Corollary}

\newcommand{\task}{\sigma}

\newcommand{\traj}{\xi}
\newcommand{\state}{x}
\newcommand{\statespace}{\mathcal{X}}
\newcommand{\safeset}{\mathcal{S}}

\newcommand{\numsafe}{N_\textrm{dem}}

\newcommand{\control}{u}
\newcommand{\controlset}{\mathcal{U}}

\newcommand{\trajxu}{\traj_{xu}}
\newcommand{\trajx}{\traj_\state}
\newcommand{\traju}{\traj_\control}

\newcommand{\constraintspace}{\mathcal{C}}

\newcommand{\cstate}{\kappa}

\newcommand{\demj}{\traj_\textrm{loc}^{j}}

\newcommand{\eq}{\textrm{eq}}
\newcommand{\ineq}{\textrm{ineq}}

\newcommand{\stat}{\mathbf{s}}
\newcommand{\phisep}{\phi_\textrm{sep}}
\newcommand{\safeprob}{\delta}
\newcommand{\Prob}{\textrm{Pr}}
\newcommand{\plan}{\textrm{plan}}
\newcommand{\ttight}{\mathbf{t}_\textrm{tight}}
\newcommand{\tnontight}{\mathbf{t}_{\neg\textrm{tight}}}
\newcommand{\trobust}{\mathbf{t}_{\textrm{rob}}}

\makeatletter
\renewcommand{\fnum@figure}{\small Fig. \thefigure}
\makeatother
\usepackage{chngcntr}
\usepackage{apptools}
\AtAppendix{\counterwithin{lemma}{section}}
\AtAppendix{\counterwithin{assumption}{section}}
\AtAppendix{\counterwithin{theorem}{section}}
\AtAppendix{\counterwithin{corollary}{section}}

\title{\LARGE \bf
Gaussian Process Constraint Learning for Scalable \\ Chance-Constrained Motion Planning from Demonstrations
}

\author{Glen Chou*, Hao Wang*, and Dmitry Berenson
\thanks{*G. Chou and H. Wang contributed equally to this work.}
\thanks{All authors are affiliated with the University of Michigan, Ann Arbor, MI 48109, USA. $\texttt{\{gchou, haowwang, dmitryb\}@umich.edu}$}%
}
\begin{document}

\maketitle
\thispagestyle{empty}
\pagestyle{empty}

\begin{abstract}

We propose a method for learning constraints represented as Gaussian processes (GPs) from locally-optimal demonstrations. Our approach uses the Karush-Kuhn-Tucker (KKT) optimality conditions to determine where on the demonstrations the constraint is tight, and a scaling of the constraint gradient at those states. 
We then train a GP representation of the constraint which is consistent with and which generalizes this information. We further show that the GP uncertainty can be used within a kinodynamic RRT to plan probabilistically-safe trajectories, and that we can exploit the GP structure within the planner to exactly achieve a specified safety probability. We demonstrate our method can learn complex, nonlinear constraints demonstrated on a 5D nonholonomic car, a 12D quadrotor, and a 3-link planar arm, all while requiring minimal prior information on the constraint. Our results suggest the learned GP constraint is accurate, outperforming previous constraint learning methods that require more \textit{a priori} knowledge.%
\end{abstract}

\section{Introduction}
The need for robots that can safely perform tasks in unstructured environments has increased as robots are deployed in the real world. One popular paradigm for teaching robots tasks is learning from demonstration (LfD) \cite{lfd_survey,apprentice_learning} via inverse optimal control (IOC), which assumes the demonstrator is solving an \textit{unconstrained} optimization, and learns the underlying reward/cost function. However, hard constraints are crucial for safety-critical applications and are not well-enforced by these methods \cite{ijrr}. To address safety in LfD, recent work has represented tasks as \textit{constrained} optimization problems, and learns the unknown cost function and constraints from demonstrations \cite{ral, wafr, corl, menner, toussaint} via the Karush-Kuhn-Tucker (KKT) optimality conditions, enabling constraint learning for complex manipulation and mobile robotics tasks. However, these methods require that the unknown constraints can be described by an \textit{a priori} known representation or parameterization (e.g. as a union of axis-aligned boxes \cite{ral, corl20}), restricting these methods to the learning of highly-structured constraints. Moreover, such representations can be highly inefficient (e.g. many boxes may be required to approximate complex constraints), leading to a computational burden that makes it challenging to scale these methods up to learn realistic constraints. Consider a demonstrator steering a quadrotor to avoid collisions with a tree (Fig. \ref{fig:tree}). On the one hand, we are unlikely to obtain an efficient constraint representation for learning trees \textit{a priori} unless we can learn one (e.g. via deep learning) using an enormous number of demonstrations, and on the other hand, a prohibitive number of boxes is needed to represent the tree.

\begin{figure}[h]
\centering
\includegraphics[width=\linewidth]{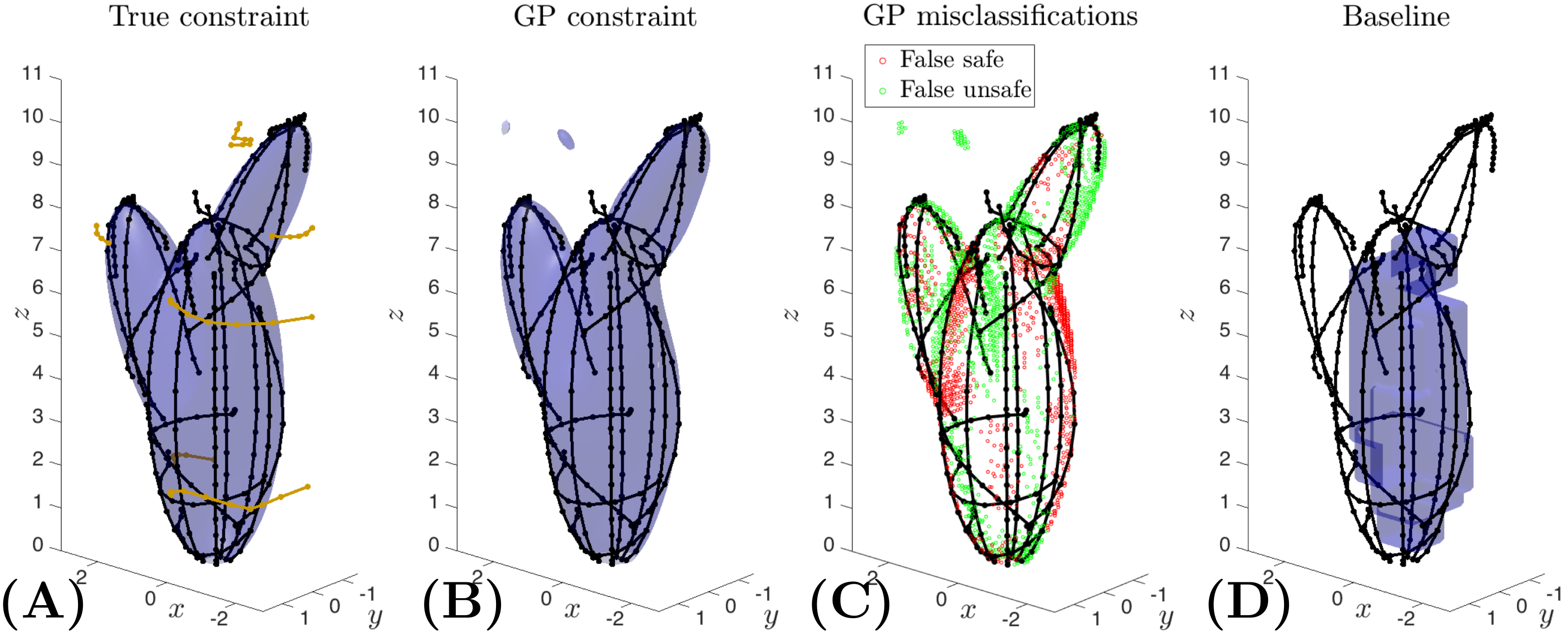}
\caption{Demonstrations (black) avoiding a tree-like obstacle on a 12D quadrotor. \textbf{(A)} True constraint (blue); plans using the GP constraint (gold). \textbf{(B)} Posterior mean of the GP constraint (blue). \textbf{(C)}  Errors of GP posterior mean w.r.t. the true constraint. \textbf{(D)} Constraint learned via \cite{ral} using 6 boxes. \vspace{-22pt}}\label{fig:tree}
\vspace{-3pt}
\end{figure}

We address these issues via the insight that the demonstrations' Karush-Kuhn-Tucker (KKT) optimality conditions provide information on 1) where the unknown constraint is tight on the demonstrations, and 2) the gradient of the constraint (i.e. the surface normal on the constraint boundary) at those points. Crucially, we show that this information can be extracted in a way that is \textit{agnostic} to the chosen constraint representation. This is in sharp contrast to prior work \cite{ral}, which uses the KKT conditions to directly determine a set of constraint parameters, for a fixed constraint parameterization, which make the demonstrations satisfy their KKT conditions. This representation-agnostic constraint information can be incorporated into a flexible non-parametric Gaussian process (GP) function approximator, which enables constraint learning while requiring minimal \textit{a priori} knowledge on the underlying constraint structure. Our contributions are:
\begin{itemize}
     \item We show how to use the demonstrations' KKT conditions to extract information on the values and gradients of the unknown constraint, how to ensure it is robust to the ill-posedness of the constraint learning problem, and how it can be jointly incorporated into a GP.
     \item We show how the uncertainty of the learned GP constraint can be used to plan chance-constrained trajectories which satisfy the unknown constraint with a specified probability, and how the Gaussian structure of the uncertainty can be exploited in the planner to exactly compute trajectory safety probabilities.
     \item We evaluate our method on complex nonlinear constraint learning problems demonstrated on a 5D nonoholonomic car, a 12D quadrotor and a three-link planar arm, showing that our method outperforms baselines.
\end{itemize}

\section{Related Work} 
Our method is related to prior work in IOC \cite{toussaint, boyd, pontryagin} that uses the KKT conditions to learn an unknown cost function, assuming the constraints are known. Other IOC methods \cite{LevinePK11, FinnLA16, WulfmeierRWOP17} use flexible function approximators to learn unknown cost functions without using known features, again assuming known constraints. Our approach is complementary to these methods, as it seeks to learn the constraints.

Our work also builds upon constraint learning methods \cite{ral, wafr, corl, shah}, which often assume a known constraint parameterization to simplify the inverse problem which recovers the unknown constraint \cite{ral, corl20}. When this assumption is removed \cite{corl}, the unsafe set defined by the unknown constraint is assumed to be well-approximated by a finite union of simple unsafe sets, e.g. axis-aligned boxes \cite{corl}. However, the inverse problem scales exponentially with the number of simple sets, as each set adds binary decision variables to the optimization. This renders complex constraints infeasible to learn unless the true constraint representation is known, restricting previous methods (e.g. \cite{ral, corl20}) to learn simple constraints, e.g. unions of a few boxes \cite{ral, corl20}, or to know the parameterization \cite[Fig. 6]{wafr} \cite[Fig. 2]{ral}. Our work is also related to methods that plan using the uncertainty in the learned constraint, e.g. \cite{corl20}. However \cite{corl20} scales exponentially in the number of simple unsafe sets; in contrast, we use the uncertainty of the GP constraint to scale more gracefully.

Finally, our work relates to planning under uncertainty, where the uncertainty may arise from sensing \cite{sampling_sensing_uncertainty}, state estimation \cite{rrbt}, motion \cite{ prob_safe_planning_with_uncertain_motion}, and the environment/obstacles; our work relates most closely to this final case. \cite{prob_approach_path_planing} plans with uncertain obstacles via chance-constrained optimization and requires polytopic obstacles and linear Gaussian dynamics. Under similar assumptions, \cite{ccrrt} embeds chance constraints in a Rapidly-Exploring Random Tree (RRT) \cite{rrt}. In contrast, we assume deterministic dynamics but can handle GP-representable constraints and nonlinear dynamics.

\section{Preliminaries and Problem Statement}

\subsection{Demonstrator's problem and KKT optimality conditions}

We represent a length $T$ demonstration of a task $\task$ performed on a deterministic discrete-time nonlinear system $\state_{t+1} = f(\state_t, \control_t, t)$, $\state\in\statespace \subseteq \mathbb{R}^{n_x}$, $\control\in\controlset \subseteq \mathbb{R}^{n_u}$ as a constrained optimization over state/control trajectories $\trajxu\doteq(\trajx,\traju)\doteq[\state_1, \ldots, \state_T, \control_1,\ldots, \control_{T-1}]$:

\begin{problem}[Forward (demonstrator's) problem / task $\task$]\label{prob:fwd_prob}
\normalfont
\begin{equation*}\label{eq:fwdprob}
	\begin{array}{>{\displaystyle}c >{\displaystyle}l}
		\underset{\trajxu}{\text{minimize}} & \quad c_\task(\trajxu)\\ 
		\text{subject to} & \quad \phi(\trajxu) \in \safeset \subseteq \constraintspace \Leftrightarrow\quad \mathbf{g}_{\neg k}^*(\phi(\trajxu)) \le \mathbf{0}\\
		& \quad \bar\phi(\trajxu) \in \bar\safeset \subseteq \bar\constraintspace, \quad \phi_\task(\trajxu) \in \safeset_\task \subseteq \constraintspace_\task \\
		& \quad \quad \Leftrightarrow \quad \mathbf{h}_k(\trajxu) = \mathbf{0},\quad \mathbf{g}_{k}(\trajxu) \le \mathbf{0}
	\end{array}\hspace{-15pt}
\end{equation*}
\end{problem}

\noindent where $c_\task(\cdot)$ is a known, possibly task-dependent cost function, and $\phi(\cdot)$ maps state/control trajectories to \textit{constraint states} $\cstate$ in \textit{constraint space} $\constraintspace$ (i.e. $\cstate \in \constraintspace$), where the constraint is evaluated. For example, for an obstacle constraint, $\phi(\cdot)$ would select the position components of the states. The safe set $\safeset \subseteq \constraintspace \subseteq \mathbb{R}^{n_c}$ is defined by the unknown inequality constraint $\mathbf{g}_{\neg k}^*(\phi(\trajxu)) \le \mathbf{0}$ and is unknown to the learner. $\bar\phi(\cdot)$ and $\phi_\task(\cdot)$ map to spaces $\bar\constraintspace$ and  $\constraintspace_\task$, containing a known task-shared safe set $\bar \safeset$ and task-dependent safe set $\safeset_\task$, defined by known equality and inequality constraints $\mathbf{h}_k(\trajxu) = \mathbf{0}$, $\mathbf{g}_{k}(\trajxu) \le \mathbf{0}$. We embed the dynamics in $\bar\safeset$ and the start/goal constraints in $\safeset_\task$. We focus on unknown scalar, state-dependent, time-separable inequality constraints

\vspace{-14pt}
\begin{equation}\label{eq:time_sep}
    \hspace{0.5pt}\mathbf{g}_{\neg k}^*(\phi(\trajxu)) \le \mathbf{0}\ \Leftrightarrow\  g_{\neg k}^*(\phisep(x_t)) \le 0,\ \ \forall t = 1, ..., T,\hspace{-5pt}
\end{equation}
\vspace{-14pt}

\noindent where $\phisep: \statespace \mapsto \constraintspace$ is the time-separable counterpart of $\phi(\cdot)$, $g_{\neg k}^*: \constraintspace \mapsto \mathbb{R}$, and $\cstate_t = \phisep(\state_t)$. We note that extending to control-dependent constraints is straightforward. Moreover, we can learn the (un)safe set for an $M$-dimensional vector-valued constraint by learning $g_{\neg k}^*(\cdot) = \max_{i=1,\ldots,M} g_{i, \neg k}^*(\cdot)$.
We assume each demonstration $\traj_\textrm{loc}$ solves Prob. \ref{prob:fwd_prob} to local optimality, satisfying Prob. \ref{prob:fwd_prob}'s KKT conditions \cite{cvxbook}. With Lagrange multipliers $\boldsymbol{\lambda}$, $\boldsymbol{\nu}$, the relevant KKT conditions for the $j$th demonstration $\demj$, denoted $\textrm{KKT}(\demj)$, are:

\vspace{-10pt}\begin{subequations}\label{eq:kkt}
	\small\begin{align}
	\hspace{-25pt}\textrm{Primal feasibility:}\ \  &\textcolor{blue}{\mathbf{g}_{\neg k}^*}(\phi(\demj)) \le \mathbf{0}, \label{eq:kkt_primal3}\\[1pt]
	\hline\hspace{-25pt}\textrm{Lagrange mult.}\quad &\textcolor{blue}{\boldsymbol{\lambda}_k^j} \ge \mathbf{0}\label{eq:kkt_lag1}\\[-3pt]
	\textrm{nonnegativity:}\quad &\textcolor{blue}{\lambda_{\neg k}^{j,t}} \ge 0, \ t = 1,..., T^j\ \Leftrightarrow\ \textcolor{blue}{\boldsymbol{\lambda}_{\neg k}^j} \ge \mathbf{0}\label{eq:kkt_lag2}\\
	\hline\hspace{-3pt}\textrm{Complementary}\quad &\textcolor{blue}{\boldsymbol{\lambda}_{k}^j}\odot\mathbf{g}_{k}(\demj) = \mathbf{0}\label{eq:kkt_comp1}\\[-3pt]
	\textrm{slackness:}\quad &\textcolor{blue}{\boldsymbol{\lambda}_{\neg k}^j}\odot\mathbf{g}_{\neg k}^*(\phi(\demj)) = \mathbf{0}\label{eq:kkt_comp2}\\
	\hline\notag\\[-12pt]
	\hspace{-25pt}\textrm{Stationarity:}\quad &\nabla_{\trajxu} c_\task(\demj) + \textcolor{blue}{\boldsymbol{\lambda}_{k}^{j}}^\top \nabla_{\trajxu} \mathbf{g}_{k}(\demj)\notag
	\\[-2pt]&\quad+\textcolor{blue}{\boldsymbol{\lambda}_{\neg k}^j}^{\hspace{-4pt}\top} \nabla_{\trajxu} \textcolor{blue}{\mathbf{g}_{\neg k}^*}(\phi(\demj))\label{eq:kkt_stat}
	\\[-2pt]&\quad+\textcolor{blue}{\boldsymbol{\nu}_{k}^j}^\top \nabla_{\trajxu} \mathbf{h}_{k}(\demj) = \mathbf{0}\notag
\end{align}
\end{subequations}
\vspace{-13pt}

\noindent where $\odot$ denotes element-wise multiplication. Here, $\boldsymbol{\lambda}_{k}^j \in \mathbb{R}^{N_\ineq^j}$, $\boldsymbol{\nu}_{k}^j \in \mathbb{R}^{N_\eq^j}$, and $\boldsymbol{\lambda}_{\neg k}^j \in \mathbb{R}^{T^j}$ are vectorized Lagrange multipliers for the known inequality, known equality, and unknown inequality constraints for $\demj$, i.e. $\boldsymbol{\lambda}_{\neg k}^j = [\boldsymbol{\lambda}_{\neg k}^{j,1},\ldots \boldsymbol{\lambda}_{\neg k}^{j,T^j}]$. The blue quantities are unknown to the learner. Intuitively, \eqref{eq:kkt_primal3} enforces that $\demj$ is feasible for Prob. \ref{prob:fwd_prob} (it lies in $\safeset$ and satisfies the known constraints), that a multiplier is zero unless its associated constraint is tight \eqref{eq:kkt_lag1}-\eqref{eq:kkt_comp2}, and that its cost cannot be locally improved \eqref{eq:kkt_stat}. 

In previous work \cite{ral, corl20}, the unknown constraint is modeled as $g_{\neg k}^*(z, \theta)$, where $\theta$ are parameters for a known representation of $g_{\neg k}^*$ with a low-order dependence on $\theta$, e.g. linear $g_{\neg k}^*(z, \theta) = \theta^\top g(z)$, where $g(z)$ are known features; the constraint learning problem then reduces to finding $\theta$. In contrast, we do not require a known parameterization for $g_{\neg k}^*(\cdot)$, instead approximating $g_{\neg k}^*(\cdot)$ as a GP to be learned.
\subsection{Overview of Gaussian processes}\label{sec:gp}
A GP is a set of (potentially infinitely many) random variables, any finite number of which have a joint Gaussian distribution \cite{gpml}. It is defined by a mean function $m(x)$ and a covariance function $k(x,x')$. In regression, GPs are often used as the prior distribution for an unknown function $f(x)$ of interest, i.e. $f\sim \mathcal{GP}(m,k)$. Given an input-output dataset $\mathcal{D} = \{(x_n,y_n)\}_{n=1}^{N_\textrm{d}}$, and assuming a noisy observation model $y_{n} \sim \mathcal{N}(f(x_n),\sigma^2)$, the predictive conditional posterior $\tilde{f}|\mathcal{D}$ is also a Gaussian if a GP is used as the prior. In performing inference at a set of points $\{z_m\}_{m=1}^{N_\textrm{q}}$, the posterior mean and covariance on these points are given by 

\vspace{-5pt}
\begin{equation}\label{eq:gp_mean}
    \mathbb{E}[\tilde{f}(\mathbf{Z})|\mathcal{D}] = k(\mathbf{Z},\mathbf{X})(k(\mathbf{X},\mathbf{X}) + \sigma^2\mathbf{\textit{I}})^{-1}\mathbf{Y},
\end{equation}

\vspace{-18pt}
\begin{equation}\label{eq:gp_cov}\footnotesize
    \textrm{cov}(\tilde{f}(\mathbf{Z})|\mathcal{D}) = k(\mathbf{Z},\mathbf{Z}) - k(\mathbf{Z},\mathbf{X})(k(\mathbf{X},\mathbf{X}) + \sigma^2\mathbf{\textit{I}})^{-1}k(\mathbf{X},\mathbf{Z}).
\end{equation}
\vspace{-13pt}

\noindent where $\mathbf{Z}$, $\mathbf{X}$, and $\mathbf{Y}$ are vectors containing all elements in $\{z_m\}_{m=1}^{N_\textrm{q}}$, $\{x_n\}_{n=1}^{N_\textrm{d}}$, and $\{y_n\}_{n=1}^{N_\textrm{d}}$, respectively \cite{gpml}.

\subsection{Problem statement}\label{sec:problem_statement}

Given locally-optimal demonstrations $\{\demj\}_{j=1}^{\numsafe}$, we want an estimate $g_{\neg k}(\cdot)$ of the unknown constraint $g_{\neg k}^*(\cdot)$, defining
\begin{equation}\label{eq:safeset}
    \safeset = \{\phisep(\state) \mid g_{\neg k}(\phisep(\state)) \le 0\} = \{\cstate \mid g_{\neg k}(\cstate) \le 0\}
\end{equation}
 \noindent as a safe set that is consistent with the demonstrations' KKT conditions, where the true constraint $g_{\neg k}^*(\cdot)$ is assumed to be a sample from $\mathcal{GP}(m,k)$. Moreover, we wish to use the learned constraint to plan trajectories $\trajxu^\plan$ that connect novel start/goal states while satisfying $g_{\neg k}^*(\cdot)$ with at least some specified probability $1-\safeprob$, i.e. $\Prob(\mathbf{g}_{\neg k}^*(\phi(\trajxu^\plan)) \le \mathbf{0}) \ge 1-\delta$.
 
\section{Method}

Our method determines where the unknown constraint is tight and its gradient at those points from the KKT conditions (Sec. \ref{sec:method_kkt}), uses this information to train a GP representation of the constraint (Sec. \ref{sec:method_gp}), and plans novel probabilistically-safe trajectories using the learned constraint (Sec. \ref{sec:method_planning}). We overview the flow of our method in Fig. \ref{fig:flow}.

\begin{figure}
\centering
\includegraphics[width=\linewidth]{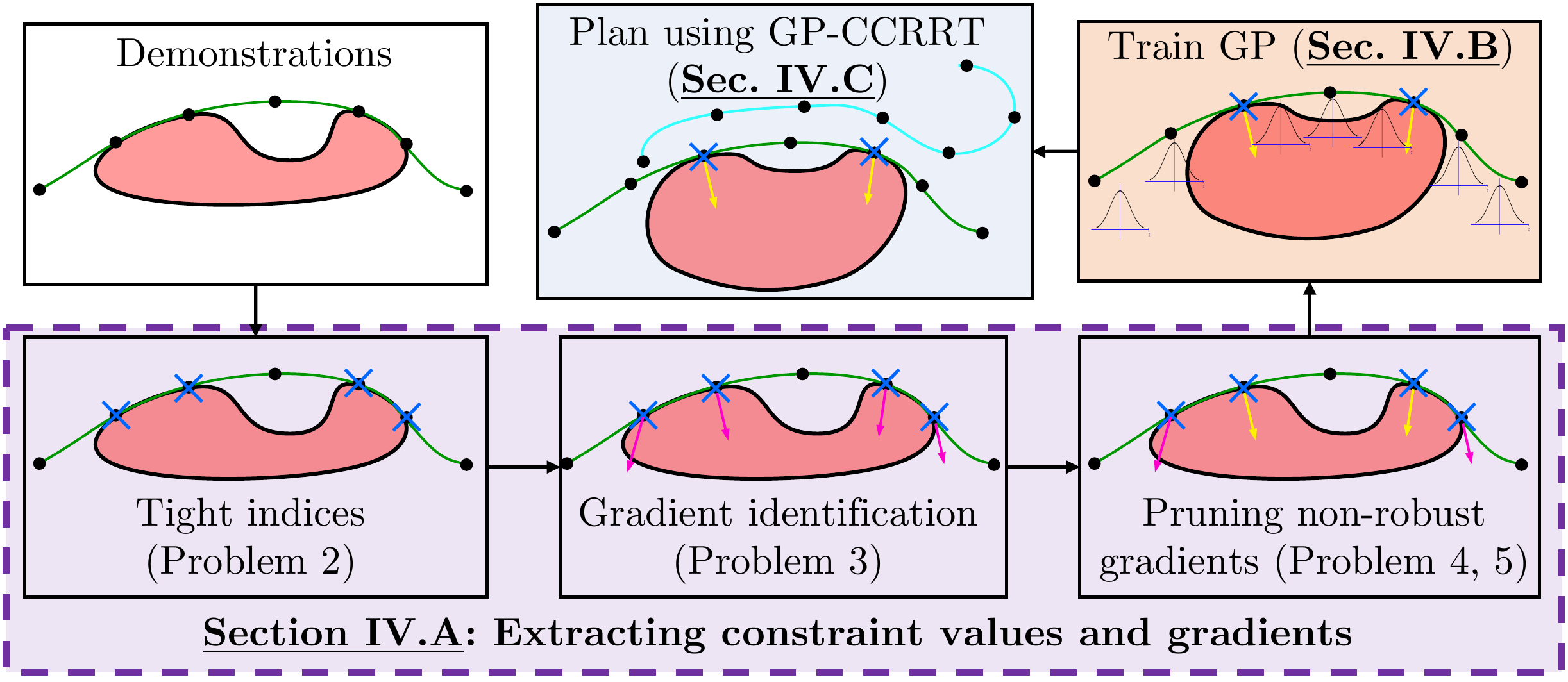}
\caption{Method flow. Given a set of locally-optimal demonstrations, we first find consistent constraint values and gradients (Sec. \ref{sec:method_kkt}), then use this data to train a consistent GP constraint representation (Sec. \ref{sec:method_gp}), and then finally plan probabilistically-safe trajectories using the learned GP constraint (Sec. \ref{sec:method_planning}).}\label{fig:flow}
\end{figure}

\subsection{Obtaining constraint value and gradient information}\label{sec:method_kkt}

For a locally-optimal demonstration, the KKT conditions \eqref{eq:kkt} provide information on the following:
\begin{enumerate}[A)]
    \item If/when the unknown constraint $g_{\neg k}^*(\cdot)$ is \textit{tight} (i.e. at which time steps of the demonstration  $g_{\neg k}^*(\phisep(\state_t)) = 0$) via complementary slackness \eqref{eq:kkt_comp2} and stationarity \eqref{eq:kkt_stat}.
    \item How the constraint changes locally around these tight demonstration points, in the form of the constraint gradient at that point, $\nabla_{\state_t} g_{\neg k}^*(\phisep(\state_t))$, via stationarity \eqref{eq:kkt_stat}. 
\end{enumerate}
Combining both sources of information is crucial in recovering an accurate constraint that is KKT-consistent.

\subsubsection{Constraint value information} We first describe a method for inferring when the unknown constraint $g_{\neg k}^*(\cdot)$ is tight. As shorthand, let the stationarity residual $$\stat^j(\boldsymbol{\lambda}_{k}^{j}, \boldsymbol{\lambda}_{\neg k}^{j}, \boldsymbol{\nu}_{k}^{j}) = \begin{bmatrix}\stat^j_{\state_1}(\boldsymbol{\lambda}_{k}^{j}, \boldsymbol{\lambda}_{\neg k}^{j}, \boldsymbol{\nu}_{k}^{j}), \allowbreak \\ \vdots \\ \stat^j_{\state_T}(\boldsymbol{\lambda}_{k}^{j}, \boldsymbol{\lambda}_{\neg k}^{j}, \boldsymbol{\nu}_{k}^{j}) \\ \stat^j_{\control_1}(\boldsymbol{\lambda}_{k}^{j} \boldsymbol{\lambda}_{\neg k}^{j}, \boldsymbol{\nu}_{k}^{j}) \\ \vdots \\ \stat^j_{\control_t}(\boldsymbol{\lambda}_{k}^{j} \boldsymbol{\lambda}_{\neg k}^{j}, \boldsymbol{\nu}_{k}^{j})\end{bmatrix} \in \mathbb{R}^{|\xi_{xu}|}$$ be the LHS of the stationarity condition \eqref{eq:kkt_stat} for the $j$th demonstration $\demj$, where $\stat^j_{\state_t / \control_t}(\boldsymbol{\lambda}_{k}^{j}, \boldsymbol{\lambda}_{\neg k}^{j}, \boldsymbol{\nu}_{k}^{j}) \in \mathbb{R}^{n_x} / \mathbb{R}^{n_u}$ is the sub-vector containing the residual terms for $\state_t$ / $\control_t$.
Recall that complementary slackness \eqref{eq:kkt_comp2} enforces that at each timestep, the Lagrange multiplier for the the unknown constraint is zero unless the constraint is tight. Moreover, as any locally-optimal trajectory $\traj_\textrm{loc}$ must satisfy the stationarity condition \eqref{eq:kkt_stat}, we can determine that the unknown constraint $g_{\neg k}^*(\cdot)$ must be tight on $\demj$ at timestep $t$ if we cannot force the norm of the stationarity residual at that timestep to be zero, i.e. $\Vert \stat^j_{\state_t}\Vert > 0$, while also enforcing that $g_{\neg k}^*(\phisep(\state_t))$ is not tight (cf. Fig. \ref{fig:label_shape}.A for intuition) and that the KKT conditions for the known constraints are satisfied. This is achieved by solving Prob. \ref{prob:tight} -- a rapidly-solvable linear program (LP):
\begin{problem}[Tightness check at time $t$ on demonstration $j$]\label{prob:tight}\normalfont
\begin{equation*}
	\hspace{-6pt}\begin{array}{>{\displaystyle}c >{\displaystyle}l >{\displaystyle}l}
				&\\[-15pt]
		\underset{\boldsymbol{\lambda}_{k}^j, \boldsymbol{\nu}_{k}^j}{\text{minimize}} & \big \Vert \stat^j_{\state_t}(\boldsymbol{\lambda}_{k}^{j}, \mathbf{0}, \boldsymbol{\nu}_{k}^{j}) \Vert_1 \\
		\text{subject to} & \eqref{eq:kkt_lag1}, \eqref{eq:kkt_comp1},\\[2pt]
	\end{array}\hspace{-15pt}
\end{equation*}
\end{problem}

\noindent where the effect of the unknown inequality constraint on the residual is erased by zeroing out its corresponding Lagrange multipliers $\boldsymbol{\lambda}_{\neg k}^j$. Then, the following result holds:
\begin{corollary}\normalfont\label{thm:tightness}
If the optimal value of Prob. \ref{prob:tight}, denoted $p_2^{t,j,*}$, is greater than $0$, then the true constraint is tight: $g_{\neg k}^*(\cstate_t^j) = 0$.
\end{corollary}
\begin{proof}\normalfont
Suppose for contradiction that $g_{\neg k}^*(\cstate_t^j) < 0$. Then, since $\demj$ satisfies \eqref{eq:kkt}, $g_{\neg k}^*(\cstate_t^j) < 0$ implies via \eqref{eq:kkt_comp2}-\eqref{eq:kkt_stat} that there exists $\boldsymbol{\lambda}_{k}^j$, $\boldsymbol{\lambda}_{\neg k}^j = \mathbf{0}$, $\boldsymbol{\nu}_{k}^j$ such that $p_2^{t,j,*} = 0$. However, by the theorem statement, $p_2^{t,j,*} > 0$. Contradiction.
\end{proof}
\noindent By solving Prob. \ref{prob:tight} and checking if $p_2^{t,j,*} > 0$ for $t = 1,\ldots,T^j$, we can find a set of timesteps where $g_{\neg k}^*(\cstate_t^j) = 0$; call these identified tight timesteps $\ttight^j$. Intuitively, Prob. \ref{prob:tight} checks if we can ensure $g_{\neg k}^*(\cstate_t^j) = 0$ despite the known constraints, e.g. dynamics, control constraints, which may be simultaneously tight. By solving Prob. \ref{prob:tight} $\sum_{j=1}^{N_\textrm{dem}} T^j$ times (once for each timestep), we can check tightness over the entire dataset. We close with two important remarks. First, complementary slackness and stationarity do not provide information on $g_{\neg k}(\cstate_t^j)$ for timesteps $\tnontight^j \doteq \{1,\ldots,T^j\} \setminus \ttight^j$; we can only deduce using primal feasibility that $g_{\neg k}^*(\cstate_t^j) \le 0$ for $t \in \tnontight^j$. Second, the estimated set of tight timesteps $\ttight$ may only be a subset of the true set of tight timesteps, i.e. $\ttight^j \subseteq \{ t \mid g_{\neg k}^*(\cstate_t^j) = 0\}$; this is because the system may lie on the constraint boundary but the cost cannot be improved by crossing it (Fig. \ref{fig:label_shape}.C), i.e. there are multipliers such that $\Vert \stat^j_{\state_t}(\boldsymbol{\lambda}_{k}^{j}, \mathbf{0}, \boldsymbol{\nu}_{k}^{j}) \Vert_1 = 0$ despite $\boldsymbol{\lambda}_{\neg k}^{j,t} = 0$; KKT cannot guarantee that these points are tight.

\begin{figure*}
\centering
\includegraphics[width=0.9\linewidth]{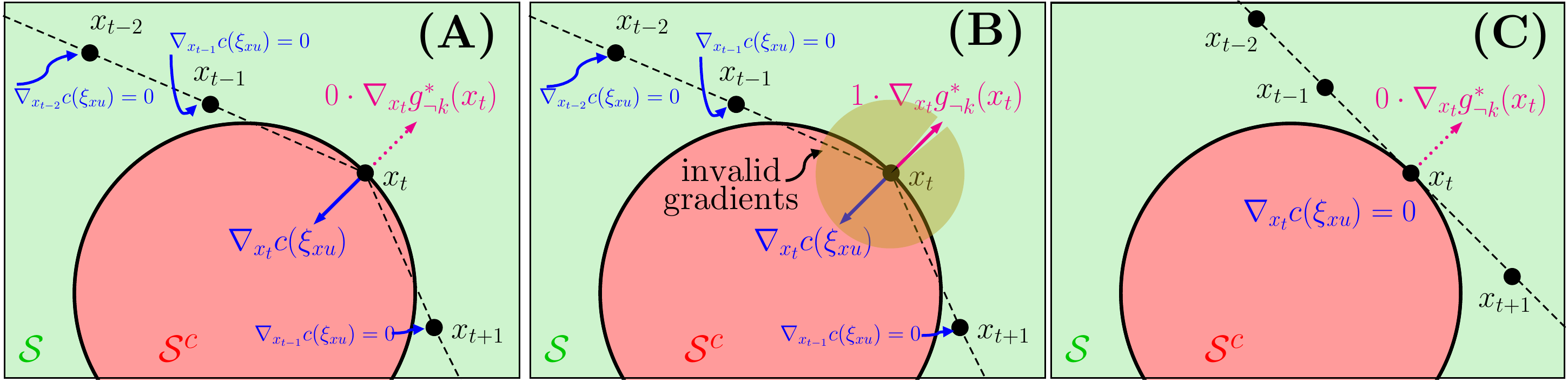}
\caption{Consider a demonstrator minimizing path length on a kinematic system; $\phisep(\state_t) = \state_t \in \mathbb{R}^2$. In this simplified setting, we can interpret \eqref{eq:kkt_stat} as balancing between vectors $\nabla c$ and $\lambda \nabla g_{\neg k}$; if they cancel to \textbf{0}, stationarity holds. We visualize this for Prob. \ref{prob:tight}-\ref{prob:label}. \textbf{(A)} Prob. \ref{prob:tight}: $\Vert\stat^{\state_t}\Vert$ can only go to zero if $\boldsymbol{\lambda}_{\neg k}^t > 0$; thus, we detect $g_{\neg k}^*(\state_t) = 0$. \textbf{(B)} Prob. \ref{prob:label}: only a scaling of the magenta constraint normal can make $\Vert \stat^{\state_t} \Vert = 0$; all gradients in gold are are not anti-parallel to $\nabla c$ and cannot cancel it. \textbf{(C)}: sometimes if $g_{\neg k}^*(\state_t) = 0$, it is still possible for $\Vert \stat^{\state_t}\Vert = 0$ with $\boldsymbol{\lambda}_{\neg k}^t = 0$.\vspace{-5pt}}
\label{fig:label_shape}
\end{figure*}

\subsubsection{Gradient value information} Next, we obtain a set of KKT-consistent gradients of the unknown constraint at each identified tight timestep $t \in \ttight^j$. In Prob. \ref{prob:label}, we set the Lagrange mutipliers $\boldsymbol{\lambda}_{\neg k}^{j,t} = 1$, for all $t \in \ttight^j$, and set the non-tight Lagrange multipliers as $\boldsymbol{\lambda}_{\neg k}^{j,t} = 0$, for all $t \in \tnontight^j$; denote the concatenation of the multipliers as $\mathbf{1}_\textrm{tight}(\demj)$. We then solve for gradients $\nabla_{\state_t} g_{\neg k}(\phisep(\state_t^j))$, for all $t \in \ttight^j$, which make the demonstration KKT-consistent along with the Lagrange multipliers of the known constraints:
\begin{problem}[Gradient identification on demonstration $j$]\label{prob:label}\normalfont
\begin{equation*}
	\hspace{-6pt}\begin{array}{>{\displaystyle}c >{\displaystyle}l >{\displaystyle}l}
				&\\[-16pt]
		\text{find} & \boldsymbol{\lambda}_{k}^j, \boldsymbol{\nu}_{k}^j,\nabla_{\state_t} g_{\neg k}(\phisep(\state_t^j)), \forall t \in \ttight^j \\
		\text{subject to} & \eqref{eq:kkt_primal3}, \eqref{eq:kkt_lag1}, \eqref{eq:kkt_comp1}\\[2pt]
		& \stat^j(\boldsymbol{\lambda}_{k}^{j}, \mathbf{1}_\textrm{tight}(\demj), \boldsymbol{\nu}_{k}^{j}) = \mathbf{0}.
	\end{array}\hspace{-15pt}
\end{equation*}
\end{problem}

\noindent Prob. \ref{prob:label} remains an LP as we fix $\boldsymbol{\lambda}_{\neg k}^j$ to avoid bilinearity. To show this does not overly restrict the set of KKT-consistent gradients, we show that while the true gradient may be not be feasible for Prob. \ref{prob:label}, a \textit{positive scaling} of it will be. A scaled gradient is acceptable for two reasons. First, it can be impossible to uniquely identify an unscaled gradient via KKT alone: by letting the tight multipliers $\boldsymbol{\lambda}_{\neg k}^{j,t}$, $t \in \ttight^j$ take positive non-unit values, they can scale their values to satisfy KKT for different scalings of $\nabla_{\state_t} g_{\neg k}(\phisep(\state_t^j))$. Second, while a scaled gradient affects how quickly the constraint changes away from the tight point, it does not affect the shape of the constraint (i.e. it does not rotate the unit surface normal vector at the boundary of the unsafe set). Let $\mathcal{F}$ be the feasible set of Prob. \ref{prob:label} and $\textrm{proj}_{\nabla g_{\neg k}}(\mathcal{F}) \doteq \{\nabla g_{\neg k} \mid \exists (\boldsymbol{\lambda}_k, \boldsymbol{\nu}_k, \nabla g_{\neg k}) \in \mathcal{F}\}$. Then, we have the following result:
\begin{theorem}\normalfont\label{thm:scaling}
A positive scaling of the true constraint gradient $\alpha_t^j \nabla_{\state_t}g_{\neg k}^*(\phisep(\state_t^j))$, for $\alpha_t^j > 0$, for all $t \in \ttight^j$, is contained in $\textrm{proj}_{\nabla g_{\neg k}}(\mathcal{F})$.
\end{theorem}
\begin{proof}
Since $\demj$ is locally-optimal, it satisfies its KKT conditions; i.e. there exists $\boldsymbol{\lambda}_k^j \ge \mathbf{0}$, $\boldsymbol{\nu}_k^j$, and $\boldsymbol{\lambda}_{\neg k}^j \ge \mathbf{0}$, where $\boldsymbol{\lambda}_{\neg k}^{j,t} = 0$, for all $t \in \tnontight^j$. This is via Prob. \ref{prob:tight}: if $t\in\tnontight^j$, the KKT conditions can be satisfied if $\boldsymbol{\lambda}_{\neg k}^{j,t} = 0$. Denote one such KKT-consistent set of multipliers as $\boldsymbol{\lambda}_k^{j,*}$, $\boldsymbol{\nu}_k^{j,*}$, and $\boldsymbol{\lambda}_{\neg k}^{j,*}$. As $g_{\neg k}^*(\cdot)$ is state-dependent and time-separable, we can write $\boldsymbol{\lambda}_{\neg k}^{j^\top} \nabla_{\trajxu} \mathbf{g}_{\neg k}^*(\phi(\demj)) = [\boldsymbol{\lambda}_{\neg k}^{j,1}\nabla_{\state_1} g_{\neg k}^*(\phisep(\state_1^j)), \ldots, \allowbreak\boldsymbol{\lambda}_{\neg k}^{j,T}\nabla_{\state_T} g_{\neg k}^*(\phisep(\state_T^j)), \allowbreak \mathbf{0}_{1\times n_u(T^j-1)}]$. Then, a feasible solution for Prob. \ref{prob:label} is $\boldsymbol{\lambda}_k^j = \boldsymbol{\lambda}_k^{j,*}$, $\boldsymbol{\nu}_k^j = \boldsymbol{\nu}_k^{j,*}$, and $\nabla_{\state_t} g_{\neg k}(\phisep(\state_t^j)) = \boldsymbol{\lambda}_{\neg k}^{j,*,t} \nabla_{\state_t} g_{\neg k}^*(\phisep(\state_t^j))$, for all $t \in \ttight^j$. Thus, the theorem holds by setting $\alpha_t^j = \boldsymbol{\lambda}_{\neg k}^{j,*,t}$.
\end{proof}

 While fixing $\boldsymbol{\lambda}_{\neg k}^j$ restricts $\textrm{proj}_{\nabla g_{\neg k}}(\mathcal{F})$ and reduces scaling ambiguity, due to other active constraints, these recovered KKT-consistent gradients may still not be unique. While scaled gradients are tolerable, a rotation of the true gradient can also lie in $\textrm{proj}_{\nabla g_{\neg k}}(\mathcal{F})$, complicating the learning as: 1) the unsafe set shape becomes uncertain, 2) modeling gradient uncertainty is challenging, as determining the set of all consistent gradient vectors is computationally intensive \cite{corl20}, and 3) the gradient uncertainty cannot be well-modeled by a Gaussian distribution, as required by our GP representation. 

Though quantifying the uncertainty in the constraint gradients is challenging, we can efficiently check if a given KKT-consistent normal vector is unique, modulo a positive scaling. This can be done by checking that there does not exist another KKT-consistent normal vector that either A) lies in the orthogonal complement of the given normal vector or B) points in directly the opposite direction (see Fig. \ref{fig:uniqueness_check}). Let $\nabla_{\state_t} \tilde g_{\neg k}^j$ be the gradient returned by Prob. \ref{prob:label} for timestep $t$ on $\demj$ and $\nabla_{\state_t} \tilde g_{\neg k}^{j,\perp} \in \mathbb{R}^{n_c \times (n_c - 1)}$ as a basis for its orthogonal complement. Then, condition A) can be checked by solving:

\begin{problem}[Orthogonal check at time $t$ on demonstration $j$]\label{prob:orthogonal}\normalfont

\begin{equation*}
	\hspace{-6pt}\begin{array}{>{\displaystyle}c >{\displaystyle}l >{\displaystyle}l}
				&\\[-18pt]
		\underset{\boldsymbol{\lambda}_{k}^j, \boldsymbol{\nu}_{k}^j,\nabla_{\state_t} g_{\neg k}(\phisep(\state_t^j))}{\text{maximize}} & \big \Vert \nabla_{\state_t} g_{\neg k}(\phisep(\state_t^j))^\top \nabla_{\state_t} \tilde g_{\neg k}^{j,\perp} \Vert_1 \\
		\text{subject to} & \eqref{eq:kkt_primal3}, \eqref{eq:kkt_lag1}, \eqref{eq:kkt_comp1}\\[2pt]
		& \stat^j(\boldsymbol{\lambda}_{k}^{j}, \mathbf{1}_\textrm{tight}(\demj), \boldsymbol{\nu}_{k}^{j}) = \mathbf{0}.
	\end{array}\hspace{-15pt}
\end{equation*}
\end{problem}

\noindent Intuitively, Prob. \ref{prob:orthogonal} searches for an alternate gradient in the orthogonal complement of the gradient obtained via Prob. \ref{prob:label} such that some assignment of multipliers also exists to satisfy the KKT conditions. Due to the non-convex objective, Prob. \ref{prob:orthogonal} can be modeled as a mixed integer linear program (MILP) with only a small number of binary variables, thus remaining rapidly-solvable. Next, condition B) can be checked via:
\begin{problem}[Anti-parallel check; time $t$ on demonstration $j$]\label{prob:opposite}\normalfont
\begin{equation*}
	\hspace{-6pt}\begin{array}{>{\displaystyle}c >{\displaystyle}l >{\displaystyle}l}
				&\\[-18pt]
		\underset{\boldsymbol{\lambda}_{k}^j, \boldsymbol{\nu}_{k}^j,\nabla_{\state_t} g_{\neg k}(\phisep(\state_t^j))}{\text{minimize}} & \nabla_{\state_t} g_{\neg k}(\phisep(\state_t^j))^\top \nabla_{\state_t} \tilde g_{\neg k}^j \\
		\text{subject to} & \eqref{eq:kkt_primal3}, \eqref{eq:kkt_lag1}, \eqref{eq:kkt_comp1}\\[2pt]
		& \stat^j(\boldsymbol{\lambda}_{k}^{j}, \mathbf{1}_\textrm{tight}(\demj), \boldsymbol{\nu}_{k}^{j}) = \mathbf{0}.
	\end{array}\hspace{-15pt}
\end{equation*}
\end{problem}
\noindent Prob. \ref{prob:opposite}, an LP, searches for a KKT-consistent gradient minimizing the dot product with $\nabla_{\state_t} \tilde g_{\neg k}^j$, i.e. pointing as anti-parallel to the original gradient as possible. Thm. \ref{thm:uniqueness} shows how Probs. \ref{prob:orthogonal}-\ref{prob:opposite} can check gradient uniqueness. Denote the optimal values of Prob. \ref{prob:orthogonal} and \ref{prob:opposite} as $p_4^*$ and $p_5^*$. We have:

\begin{figure}
\centering
\includegraphics[width=\linewidth]{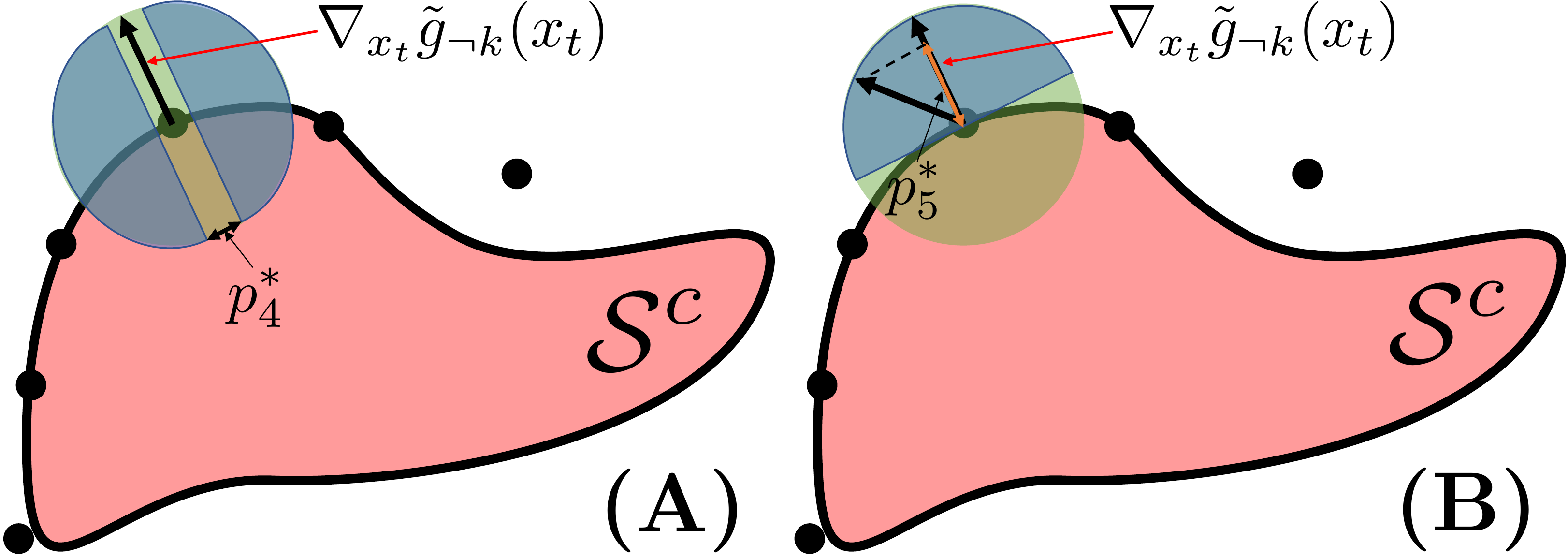}
\caption{Prob. \ref{prob:orthogonal} and \ref{prob:opposite} intuition. \textbf{(A)}: Prob. \ref{prob:orthogonal} searches for a new gradient orthogonal to the original gradient by maximizing the distance from the origin as measured in the coordinates of $\nabla_{\state_t} \tilde g_{\neg k}^\perp$. If $p_4^* = 0$ (i.e. the gradient remains in the gap between the blue areas as the gap $\rightarrow 0$), the new gradient must remain in the span of the original gradient. \textbf{(B)}: Prob. \ref{prob:opposite} searches for a new gradient with minimal dot product w.r.t. the original gradient; if the result remains in the blue semicircle (i.e. $p_5^* > 0$) and $p_4^* = 0$, the gradient from Prob. \ref{prob:label} is unique up to a scaling.}\label{fig:uniqueness_check}
\end{figure}

\begin{theorem}\normalfont\label{thm:uniqueness}
If $p_4^* = 0$ and $p_5^* > 0$, the true gradient $\nabla_{\state_t} g_{\neg k}^*(\phisep(\state_t^j))$ is a positive scaling of the recovered gradient $\nabla_{\state_t} \tilde g_{\neg k}(\phisep(\state_t^j))$, i.e. there exists $\alpha > 0$ such that $\nabla_{\state_t} g_{\neg k}^*(\phisep(\state_t^j)) = \alpha \nabla_{\state_t} \tilde g_{\neg k}(\phisep(\state_t^j))$.
\end{theorem}
\begin{proof}\normalfont
 First, $p_4^* = 0$ iff all feasible $\nabla_{\state_t} g_{\neg k}(\phisep(\state_t^j))$ lie in $\textrm{span}(\nabla_{\state_t} \tilde g_{\neg k})$, as the objective of Prob. \ref{prob:orthogonal} is just the norm of the coordinates of $\nabla_{\state_t} g_{\neg k}(\phisep(\state_t^j))$ in the basis of the orthogonal complement, i.e. there exists $\beta \in \mathbb{R}$ such that $\nabla_{\state_t} g_{\neg k}^j = \beta \nabla_{\state_t} \tilde g_{\neg k}^j$. Second, if $p_5^* > 0$, then $\nabla_{\state_t} g_{\neg k}^{j,\top} \nabla_{\state_t} \tilde g_{\neg k}^j > 0$, for all $\nabla_{\state_t} g_{\neg k}^j \in \textrm{proj}_{\nabla_{g_{\neg k}}}(\mathcal{F})$. Third, by combining these two results, we have that $\beta \nabla_{\state_t} \tilde g_{\neg k}^{j,\top} \nabla_{\state_t} \tilde g_{\neg k}^j > 0$, implying $\beta > 0$, as $\Vert \nabla_{\state_t} \tilde g_{\neg k}^j\Vert > 0$ in order for $t \in \ttight^j$. Finally, from Thm. \ref{thm:scaling}, we know $\nabla_{\state_t} g_{\neg k}^* = \gamma \nabla_{\state_t} g_{\neg k}$, for some $\gamma > 0$ and $\nabla_{\state_t} g_{\neg k} \in \textrm{proj}_{\nabla_{g_{\neg k}}}(\mathcal{F})$; we recover the theorem statement by setting $\alpha = \beta\gamma$.
\end{proof}
Our approach is to use the tight points with a unique KKT-consistent unit normal vector to train our GP constraint (see Sec. \ref{sec:method_gp}); we call these gradients \textit{robustly-identified} and their timesteps as $\trobust^j \subseteq \ttight^j$, for all $j = 1,\ldots,N_\textrm{dem}$.

\subsection{Embedding KKT-based information in a Gaussian process}\label{sec:method_gp}
Let the number of robustly-identified points over all demonstrations be $N_\textrm{robust}$. We collect the constraint states corresponding to the robustly-identified gradients and denote it as $\mathcal{D}_\cstate \doteq \{\phisep(\state_t^j) \mid t \in \trobust^j, j \in \{1, \ldots, N_\textrm{dem} \}\} \in \mathbb{R}^{N_\textrm{robust} \times n_c}$. We also collect the robustly-identified gradients $\mathcal{D}_\nabla \doteq \{\nabla_{\state_t}g_{\neg k}(\phisep(\state_t^j)) \mid t \in \trobust^j, j \in \{1, \ldots, N_\textrm{dem}\}\} \in \mathbb{R}^{N_\textrm{robust} \times n_c}$. Moreover, as the value of the unknown constraint equals zero at all robustly-identified points, we can define a third set $\mathcal{D}_g \doteq \mathbf{0}_{N_\textrm{robust}}$, i.e. the zero vector of size $N_\textrm{robust}$. We wish to learn a GP which is consistent with both the constraint values $\mathcal{D}_g$ as well as the constraint gradients $\mathcal{D}_\nabla$. Note that derivative of a GP is a GP, and the joint distribution of a GP and its derivative is also a GP \cite{solak}; forming this joint GP provides us an avenue for incorporating both the constraint value and gradient information. Like the derivation of the GP posterior without derivative observations (Sec. \ref{sec:gp}), one can derive the posterior distribution conditioned on the training inputs, their derivatives, and the outputs. For brevity, please refer to \cite{gp_derivatives} for detailed derivations. For this joint GP, we can define the training inputs and outputs as $\mathbf{X} = \mathcal{D}_\cstate$ and $\mathbf{Y} = [\mathcal{D}_g, \mathcal{D}_\nabla]$, comprising the dataset $\mathcal{D} = (\mathbf{X}, \mathbf{Y})$, and use the negative marginal log likelihood $-\mathcal{L}_\textrm{MLL}$ \cite[Eq. 2.30]{gpml} to optimize the GP hyperparameters.

A key subtlety is that as the learned constraint is a GP, its constraint value at any given query point is not deterministic; rather, it is sampled from a Gaussian distribution whose mean and variance is determined by the training data and the location of the query point (i.e. $g_{\neg k}(\cstate_m) \sim \mathcal{N}(\mu_m,\sigma^{2}_m\mid\mathcal{D},\cstate_m))$. Moreover, while the demonstrations are guaranteed safe by assumption (i.e. $g_{\neg k}(\cstate_t) \leq 0$ for all $t$), the stochasticity of the GP values prevents us from enforcing that the demonstrations are safe with probability $1$, as the Gaussian has infinite support. Instead, we select a standard deviation threshold $\rho$ for which we want the demonstrations to be safe and add a hinge loss on its violation, where $R = \sum_{j=1}^{N_\textrm{dem}} T^j$:
\begin{equation}
    \mathcal{L}_\textrm{feas} = (1/R)\textstyle\sum_{n=1}^R \max(\mu(x_n\mid \mathcal{D}) + \rho\sigma(x_n\mid \mathcal{D}), 0).
\end{equation}

Then the full training loss is $\mathcal{L} = -\mathcal{L}_\textrm{MLL} + \mathcal{L}_\textrm{feas}$.

\subsection{Planning with the learned constraint}\label{sec:method_planning}

We describe a method for planning with the learned GP constraint. As the GP is probabilistic, so is the boundary of the learned safe set $\safeset$ \eqref{eq:safeset}; thus, our planner provides probabilistic, rather than deterministic, safety guarantees. As the dynamics are assumed known, we only consider the uncertainty of the GP constraint in planning. Recall that we wish to connect a start and goal state with a dynamically-feasible trajectory $\trajxu^\textrm{plan}$ that satisfies the true constraint $\mathbf{g}_{\neg k}^*(\phi(\trajxu^\textrm{plan}))$ with probability $1-\delta$. From the assumption (Sec. \ref{sec:problem_statement}) that $g_{\neg k}^*$ is drawn from the GP, we achieve this by satisfying the learned constraint $\mathbf{g}_{\neg k}(\phi(\trajxu^\textrm{plan}))$ with probability $1-\delta$. Unlike previous work in planning under uncertainty \cite{prob_approach_path_planing, ccrrt, ccrrt_star}, we make no structural assumptions on the dynamics or the shape of the constraints.

We modify a constrained kinodynamic RRT \cite{lavalle2006planning} to plan with the learned constraint, though our method can be adapted to other sampling or optimization-based planners. Our planner, which we refer to as Gaussian Process-Chance Constrained RRT (GP-CCRRT), is presented in Alg. \ref{alg:cdfrrt_extend}. The main novelty of the proposed planner is its GP constraint-checker: when a new node $\state_q$ is sampled, instead of checking if $\state_q$ satisfies the timestep-independent chance constraint $\Prob(g_{\neg k}(\phisep(\state_q)) \le 0) \geq 1-\delta$, we check if we can connect $\state_q$ to the tree by exactly evaluating the joint probability of safety for the full trajectory from the root to the candidate node $\state_q$ (line 7-8). Our ability to efficiently compute this probability relies on the Gaussian structure of our learned GP constraint representation. Let the full trajectory from the root to $\state_q$, denoted $\xi_q$, be length $K$. Evaluating the learned GP $g_{\neg k}(\cdot)$ at those $K$ points returns the mean and covariance matrix of the predictive posterior distribution, which is a $K$-variate Gaussian (line 7). Then, the trajectory safety probability, $p_\textrm{safe}^q \doteq \Pr(\bigwedge_{n=1}^K (g_{\neg k}(\cstate_n) \le 0))$, is obtained by integrating the density of this $|K|$-variate Gaussian from $-\infty$ to $0$ in each dimension (i.e. the cumulative distribution function, or \texttt{cdf}), evaluated at $\mathbf{0}_K$ (line 8). Highly-optimized implementations of the multivariate Gaussian CDF \cite{genz} enable fast CDF evaluation at planning time. Finally, node $\state_q$ is accepted if $\xi_q$ is safe with at least probability $1-\delta$ (line 8). We visualize the GP constraint check in Fig. \ref{fig:gp_planner}.

\begin{figure}
\centering
\includegraphics[width=\linewidth]{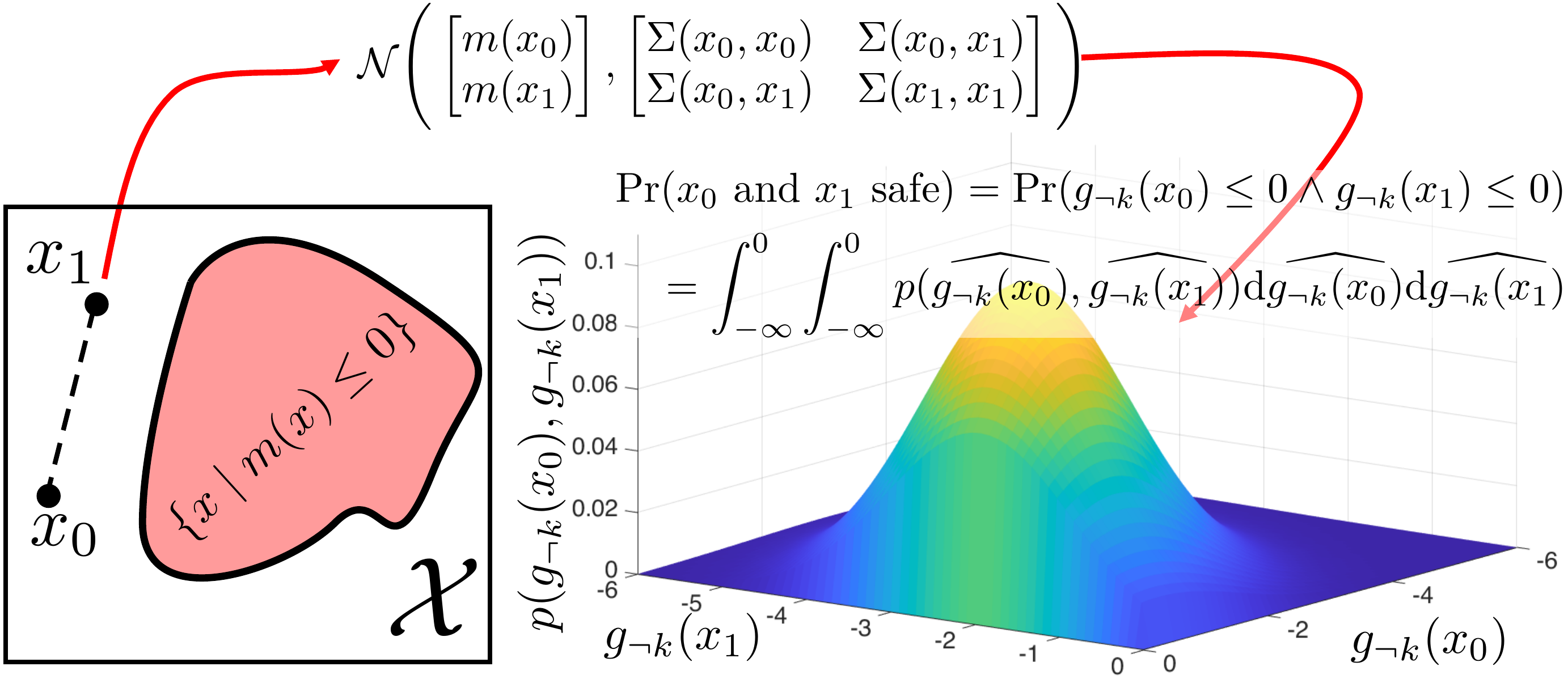}
\caption{Illustration of GP-CCRRT. A candidate length $2$ trajectory from the root of the RRT induces a bivariate Gaussian; its safety probability can then be calculated by calculating the CDF of the induced Gaussian.}\label{fig:gp_planner}
\end{figure}

\begin{algorithm}[h!] \label{alg:cdfrrt_extend}
\SetAlgoLined
\DontPrintSemicolon
\KwIn{$x_I$, $x_G$, $\epsilon$, goal\_bias}
\SetKwFunction{SampleState}{SampleState}
\SetKwFunction{SampleControl}{SampleControl}
\SetKwFunction{NearestNeighbor}{NearestNeighbor}
\SetKwFunction{Model}{Model}
\SetKwFunction{ConstructPath}{ConstructPath}
\SetKwFunction{ShootToDesired}{ShootToDesired}
\SetKwFunction{PathFromRoot}{PathFromRoot}

$\mathcal{T}_\state \leftarrow \{x_I\}, \mathcal{T}_\control \leftarrow \{\}$ \\

\While{\upshape True}{
    $\state_{\textrm{desired}} \leftarrow $ \SampleState{\upshape goal\_bias} \\
    $\state_{\textrm{near}} \leftarrow $ \NearestNeighbor{\upshape $\mathcal{T}_\state$, $x_{\textrm{desired}}$} \\
    $\xi_{\textrm{prev}} \leftarrow \PathFromRoot(\state_\textrm{near})$ \\
    $\state_q, \control_q \rightarrow \ShootToDesired(\state_\textrm{near}, \state_\textrm{desired})$\\
    $\mu, \Sigma \leftarrow g_{\neg k}(\phi(\xi_\textrm{prev} \cup \state_q))$ \\
    \lIf{\normalfont\texttt{cdf}($\mathcal{N}(\mu, \Sigma), \mathbf{0})$ $\geq 1-\delta$}{
        $(\mathcal{T}_\state, \mathcal{T}_\control)\leftarrow (\mathcal{T}_\state, \mathcal{T}_\control) \cup (\state_q, \control_q)$\hspace{-20pt} }
    \If{\normalfont $\Vert \state_q - x_G \Vert \le \epsilon$}{
        return $\trajxu^\plan \leftarrow$ \ConstructPath{\upshape $\mathcal{T}_\state$, $\mathcal{T}_\control$, $x_q$}\hspace{-15pt}
    }
}
\caption{GP-CCRRT}
\end{algorithm}

\section{Results}\label{sec:results}

We evaluate our method on learning complex, nonlinear constraints demonstrated on a point robot, nonholonomic car, quadrotor, and arm. Please see the video for visualizations. We train all GPs using GPyTorch with an RBF kernel using the Adam optimizer. We obtain demonstrations by solving Prob. \ref{prob:fwd_prob} using IPOPT \cite{ipopt}. We compare with two baselines. The first, \cite[Prob. 4]{ral}, approximates the unknown constraint as a union of $B$ axis-aligned boxes (as in \cite[Sec. 4.4]{corl}). In the second, we use a neural network (NN) instead of a GP to fit the constraint using the same data; in all examples, the NN has 5 hidden layers of size 256, 512, 1024, 512, and 128 and is trained for 200 epochs with learning rate $5 \times 10^{-5}$. To train the NN, we use MSE losses on the target tight constraint values and gradients with a hinge loss that encourages all points to be feasible. We also compute (Table \ref{table:results}) how many states are falsely claimed safe (``false safe (FS)") or unsafe (``false unsafe (FU)'') by setting $\safeset = \{\cstate \mid \mu(\cstate \mid \mathcal{D}) + \tau\sigma(\cstate \mid \mathcal{D}) \le 0 \}$ for standard deviations $\tau \in \{0, 1, 2, 2.33\}$. While $\safeset$ is not used in GP-CCRRT (it uses joint instead of individual safety probabilities), it is a good surrogate for constraint accuracy. Finally, Probs. \ref{prob:tight}-\ref{prob:opposite} are all solved in 0.5 seconds.

\begin{table}[thb!] \centering\footnotesize
\begin{tabular}{ll|llll|ll}
                      &                   &  \multicolumn{4}{c}{Our method} & \multicolumn{2}{c}{Baselines} \\
                      &                   &  $0\sigma_p$ & $1\sigma_p$ & $2\sigma_p$ & $2.33\sigma_p$ & \cite{ral} \hspace{-10pt} & NN \\ \hline
\hspace{-5pt}\multirow{2}{*}{Cup}\hspace{-15pt}  & FS (\%)   & 0.004   & 0.000   & 0.000   & 0.000      & 22.616  & 52.706 \hspace{-10pt}        \\
                      & FU (\%) & 0.022   & 1.294  & 3.532  & 4.684     & 5.284 &  0.000 \hspace{-10pt}        \\ \hline \hline
\hspace{-5pt}\multirow{2}{*}{Car}\hspace{-15pt}  & FS (\%)   & 1.741   & 0.319   & 0.071   & 0.042      & 5.947 & 15.555 \hspace{-10pt}        \\
                      & FU (\%) & 0.424   & 58.761  & 64.807  & 66.305     & 0.117 &  0.000 \hspace{-10pt}        \\ \hline \hline
\hspace{-5pt}\multirow{2}{*}{Box}\hspace{-15pt}  & FS (\%)   & 3.230   & 0.462   & 0.189   & 0.138      & 0.000 & 10.859 \hspace{-10pt}      \\
                      & FU (\%) & 1.648   & 81.146  & 86.641  & 87.190     & 0.000  &0.000  \hspace{-10pt}     \\ \hline \hline
\hspace{-5pt}\multirow{2}{*}{Tree}\hspace{-15pt} & FS (\%)   & 0.593   & 0.057   & 0.004   & 0.000      & 14.867 & 23.427  \hspace{-10pt}      \\
                      & FU (\%) & 0.729   & 11.108  & 31.412  & 37.773     & 0.160  &  0.000 \hspace{-10pt}     \\ \hline \hline
\hspace{-5pt}\multirow{2}{*}{Arm}\hspace{-15pt}  & FS (\%)   & 1.403   & 0.163   & 0.012   & 0.003      & 17.179    & 15.029 \hspace{-10pt}   \\
                      & FU (\%) & 0.658   & 57.294  & 70.644  & 73.490     & 0.808 & 0.151     \hspace{-10pt}
\end{tabular}
\caption{GP classification errors (False Safe (FS); False Unsafe (FU)).}\label{table:results}
\end{table}

\begin{figure}
\centering
\includegraphics[width=\linewidth]{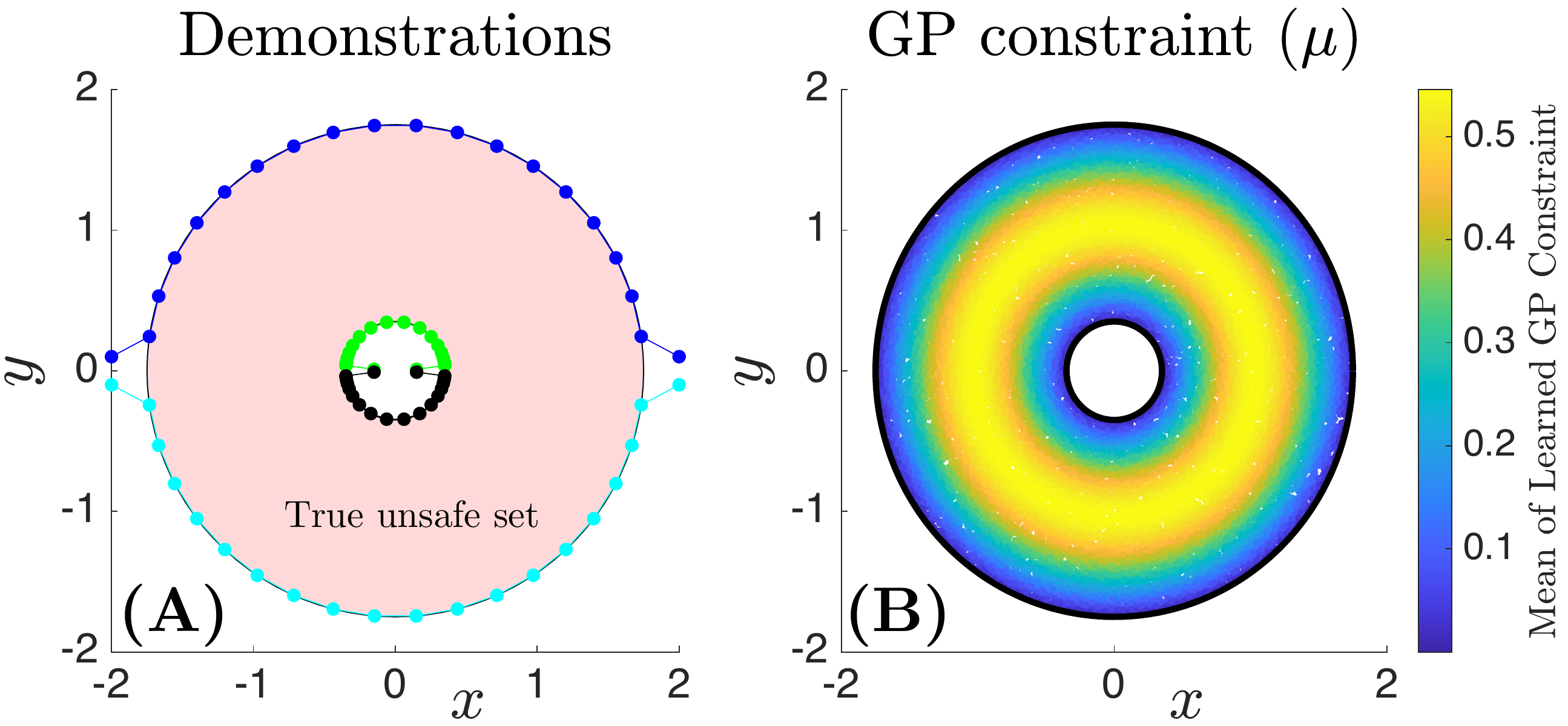}
\caption{2D hollow cup constraint. (A) Demonstrations and true constraint. (B) Learned GP posterior mean, with true constraint overlaid.}\label{fig:cup}
\end{figure}

\noindent\textbf{2D cup constraint}: The purpose of our first example is to demonstrate that our approach can learn complicated unsafe sets with hollow interiors. Consider demonstrations wiping the interior and exterior of a cup. The cup is centered at the origin and has inner and outer radii $\underline r$ and $\bar r$, respectively. One way to represent the wiping task is to minimize the cumulative distance from the center of the rim over time, i.e. $c(\xi) = \sum_{t=1}^T \Vert \Vert \chi_t \Vert - \frac{\bar r + \underline r}{2}\Vert_2^2$, subject to point-robot dynamics, control constraints, and nonpenetration with the cup, where $\chi_t = [x_t, y_t]$. In this example, we aim to learn the shape of the cup from the demonstrations (i.e. the unsafe set between the inner and outer cup radii, see Fig. \ref{fig:cup}.A). Given four demonstrations (Fig. \ref{fig:cup}.A), and by training the GP for 500 epochs at learning rate 0.05, we are able to recover the cup shape (Fig. \ref{fig:cup}.B) with very high accuracy (see Table \ref{table:results}). In contrast, neither of the baselines can accurately recover the constraint (Table \ref{table:results}); the NN fails to accurately fit the constraint gradients, while \cite{ral} fails to accurately fill the interior of the unsafe set (when allocated 5 boxes in its representation).

\begin{figure}
\centering
\includegraphics[width=\linewidth]{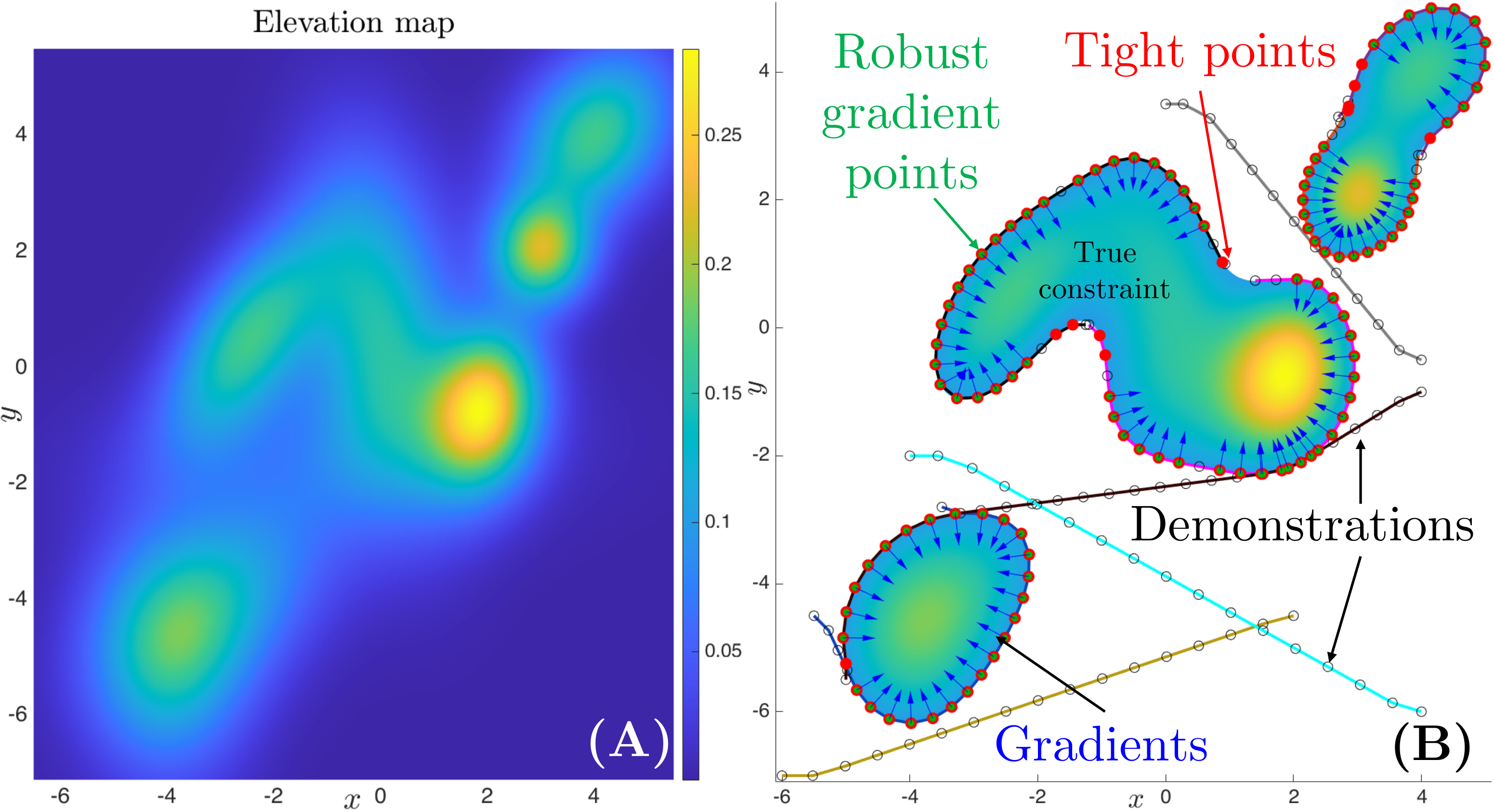}
\caption{5D car example. \textbf{(A)} Hilly terrain map. \textbf{(B)} Demonstrations; identified tight points (red); robustly-identified timesteps (green); robustly-identified gradients (blue).}\label{fig:car_id}
\end{figure}

\begin{figure}
\centering
\includegraphics[width=\linewidth]{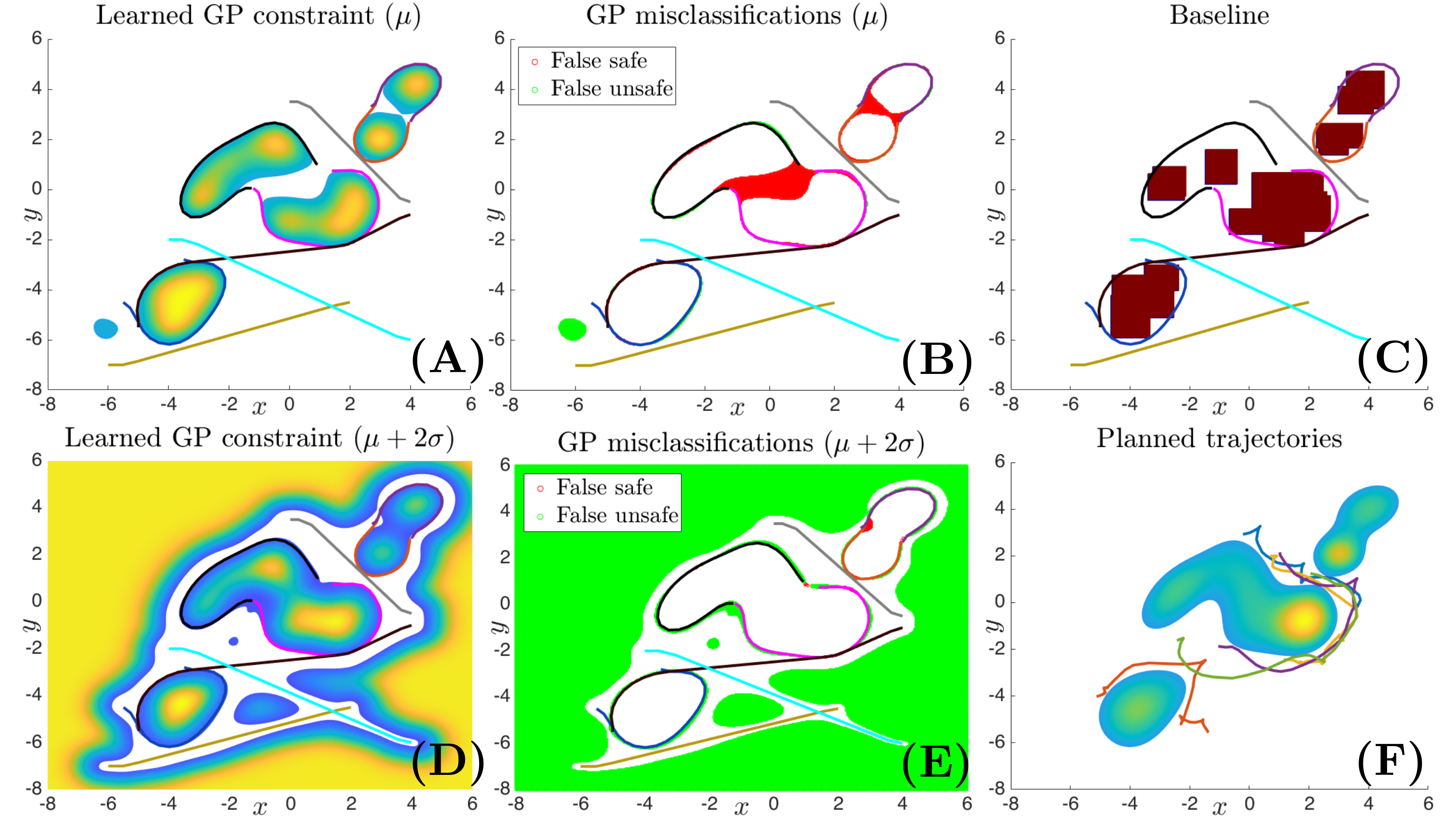}
\caption{5D car example, learned. \textbf{(A)} Learned GP constraint, mean function. \textbf{(B)} Mean function misclassifications. \textbf{(C)} Constraint learned using baseline \cite{ral}. \textbf{(D)} Learned GP constraint, buffered by GP uncertainty. \textbf{(E)} Buffered misclassifications. \textbf{(F)} Plans computed using learned GP constraint.}\label{fig:car_subplots}
\end{figure}

\noindent\textbf{5D nonholonomic car}: We first evaluate our method on a 5D car, showing that can learn a nonlinear, disconnected constraint without prior knowledge. Consider an autonomous vehicle driving on hilly terrain (Fig. \ref{fig:car_id}.A) which must stay below a maximum elevation; the corresponding unsafe set (i.e. the subset of the map above the elevation limit) is the filled-in region in Fig. \ref{fig:car_id}.B. We use the second-order unicycle dynamics from \cite[Eq. 13.46]{lavalle2006planning} with a discretization time of $\Delta T$ = 0.5. Prior work \cite{wafr} studied a similar example; however, in \cite{wafr}, the map (i.e. the constraint representation) is given, so the only the elevation threshold must be learned. In contrast, we are not given the map, and must learn the representation jointly with the threshold -- a much harder problem. 

We obtain 9 demonstrations minimizing the $xy$-space path length, and a control constraint of $\Vert u_t \Vert_2^2 \le 5$ is imposed for all time. Here, $\phisep$ maps to the $xy$-state components. By solving Prob. \ref{prob:tight}, we identify tight points (Fig. \ref{fig:car_id}.B, red). Next, by solving Probs. \ref{prob:label}-\ref{prob:opposite}, we find robustly-consistent gradients (Fig. \ref{fig:car_id}.B, blue arrows) at a subset of the tight points (Fig. \ref{fig:car_id}.B, green). Note the accuracy of Prob. \ref{prob:tight}, which identifies that the cyan trajectory is not tight, despite it curving due to dynamical constraints, and correct identification for the black trajectory, which makes and breaks contact with the constraint boundary. The few tight points that are not identified (e.g. near $[-2, 2]$) are where the constraint boundary is flat; thus, the sub-trajectory is optimal despite being on the boundary (Fig. \ref{fig:label_shape}.C). Note that most tight points also have robustly-identifiable gradients; the exceptions are before/after the system leaves the constraint boundary; this is due to the dynamics, as the car may brake/turn to prevent constraint violation, expanding the set of consistent gradients.

We train the GP for 150 epochs at learning rate $0.08$. In Fig. \ref{fig:car_subplots}, we show the GP accuracy and compare with the baseline \cite{ral}. Overall, the GP mean faithfully recreates the true unsafe set (Fig. \ref{fig:car_subplots}.A), though it misclassifies (Fig. \ref{fig:car_subplots}.B) the center of the middle and top obstacles; this is as there are few tight points in that area. Still, when buffering the GP with its uncertainty (Fig. \ref{fig:car_subplots}.D-E), the regions which are falsely classified shrink (see Table \ref{table:results}), though this is at the cost of conservatively marking much of the map far from the demonstrations as unsafe. This arises from the GP's ability to capture epistemic uncertainty and can actually be desirable as it leads to cautious plans that remain near the data and away from unseen constraints that are inactive on the demonstrations. For the baseline \cite{ral}, we use 20 boxes and terminate the optimization after 60 minutes. The result has higher error than the learned GP constraint and fails to capture the constraint shape. The NN baseline is also inaccurate, as it drives the value of most tight points to 0 but fails to fit the gradient data (Table \ref{table:results}). Finally, we plan with GP-CCRRT with a safety probability of 0.9; five plans are shown in Fig. \ref{fig:car_subplots}.F, which all satisfy $g_{\neg k}^*(\cdot)$. On average, our planner solves in 3 minutes, with 20 and 50 percent of that time being dedicated to GP posterior and CDF computation, respectively; this can be sped up via lazy checking of the CDF constraint. Overall, this example suggests we can learn nonconvex constraints with minimal \textit{a priori} knowledge.

\noindent\textbf{12D quadrotor}: We evaluate our method on two constraint learning tasks on a 12D quadrotor (see \cite[Eq. 19]{auro} for the dynamics). We first show our method achieves comparable performance with \cite{ral} for learning constraints that can be represented as a union of boxes. We are given 24 short demonstrations (Fig. \ref{fig:quad_box}.A) that minimize $xyz$ path length while satisfying a control constraint $\Vert u_t \Vert_2^2 \le 100$, for all $t$. Moreover, the baseline \cite{ral} is also provided the information that the constraint can be represented as a union of two axis-aligned boxes (thus learning the constraint exactly). We train the GP for 600 epochs at learning rate $0.1$, and the learned GP (Fig. \ref{fig:quad_box}.B) captures the union-of-boxes shape well. Two main inaccuracies are A) the interior of the learned box is hollow (this is expected, as no data can be collected in the obstacle) and B) there are some ``ringing effects" (this is caused by the GP, which favors smooth functions, attempting to fit the discontinuous box gradients). Numerically, Table \ref{table:results} shows that the GP misclassifications are low, and moreover, the number of states that are falsely claimed to be safe can be driven near zero by buffering with the GP uncertainty. The GP outperforms the NN baseline, which again struggles to fit the gradient data. Overall, this example suggests our method also performs well where previous methods excel.

Next, we evaluate our approach on learning a complex nonlinear constraint which is well beyond the capability of the baseline \cite{ral}. We are given 25 demonstrations (Fig. \ref{fig:tree}.A, black) avoiding collisions with an unknown tree-like obstacle to be learned, which is a union of three ellipsoids (Fig. \ref{fig:tree}.A, blue). Crucially, we lack \textit{a priori} knowledge on the structure of $g_{\neg k}^*(\cdot)$. The dynamics and cost function used are as in \cite{corl20} and \cite{ral}, respectively, and $\phisep$ maps to the $xyz$-state components. We train the GP for 150 epochs at learning rate $0.08$. In comparing with the baseline, we use 6 axis-aligned boxes and time out the optimization after 2 hours.

We visualize our results in Fig. \ref{fig:tree}.B-C. Our learned GP is visually accurate (Fig. \ref{fig:tree}.B), with minor errors (Fig. \ref{fig:tree}.C) where there are no tight points. This is reasonable, as we cannot expect the GP to be accurate far from the data. In contrast, the baseline \cite{ral} is inaccurate (Fig. \ref{fig:tree}.D), failing to cover the upper portion of the obstacle; moreover, the shape is inaccurate due to the limitations of axis-aligned boxes. The NN baseline is also inaccurate, failing to fit the constraint gradients (Table \ref{table:results}). Numerical results in Table \ref{table:results} suggest that the GP mean is the most accurate when considering both metrics. As before, the ``False Safe" rate can be made smaller at the cost of conservativeness by buffering with the predictive uncertainty. In contrast, the baseline has a high ``False Safe" rate, which can lead to constraint violation in planning. We show six plans computed via GP-CCRRT (Fig. \ref{fig:tree}.A, gold), which are safe with probability at least 0.9, and which are safe for the true constraint. On average, planning takes 90 seconds; 55 and 20 percent of this is due to GP posterior and CDF calculations. This example suggests our method scales to complex constraints on high-dimensional systems while requiring minimal prior information.

\begin{figure}
\centering
\includegraphics[width=\linewidth]{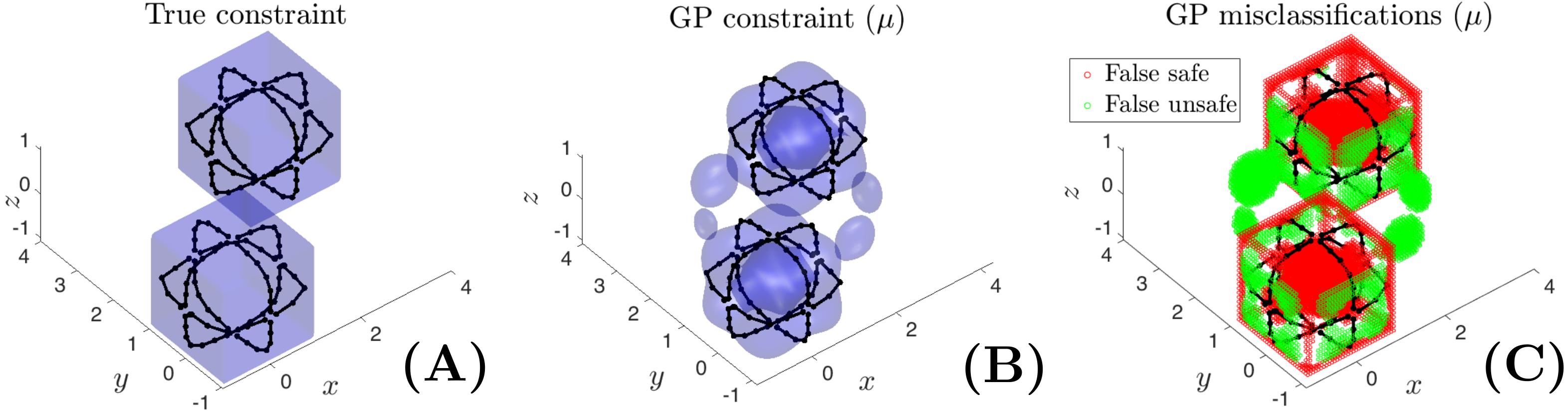}
\caption{12D quadrotor box example. \textbf{(A)} Two box obstacles; demonstrations. \textbf{(B)} Learned GP constraint (mean). \textbf{(C)} GP misclassifications (mean).}\label{fig:quad_box}
\end{figure}
\begin{figure}
\centering
\includegraphics[width=\linewidth]{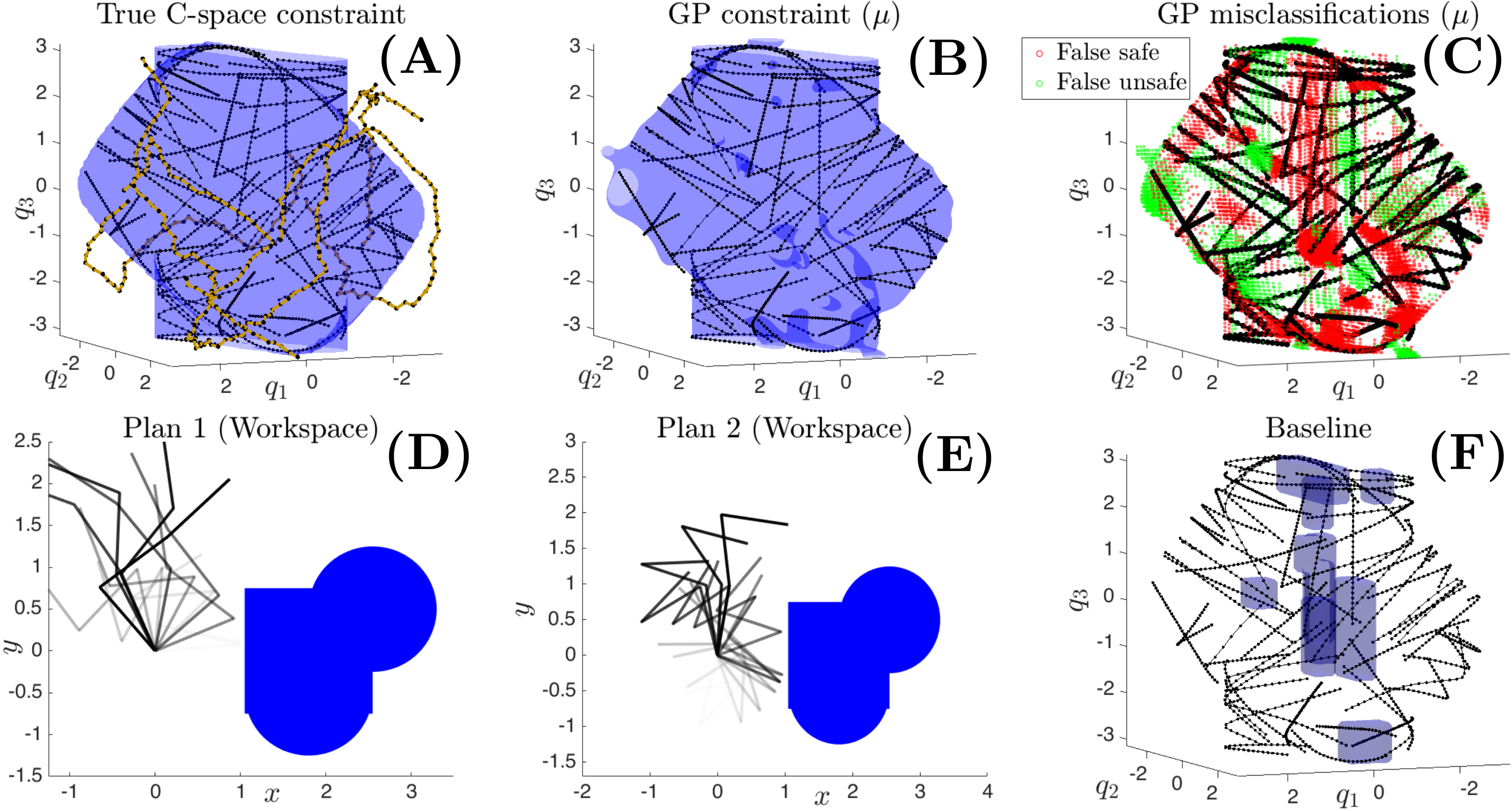}
\caption{3-link planar arm example, learned. \textbf{(A)} True C-space constraint; plans found using GP constraint (gold). \textbf{(B)} Learned GP constraint, mean function. \textbf{(C)} Mean function misclassifications. \textbf{(D-E)} Plans computed using learned GP constraint (workspace). \textbf{(F)} Constraint learned with baseline \cite{ral}.}\label{fig:arm}
\end{figure}

\noindent\textbf{Planar 3-link manipulator}: Finally, we evaluate our method on a kinematic planar 3-link arm. The arm is mounted at the origin and must avoid a blue nonconvex workspace obstacle (Fig. \ref{fig:arm}.D, E). To show that our method can learn complex constraints on articulated robots, we learn the configuration space (C-space) representation of the obstacle (Fig. \ref{fig:arm}.A, blue), which is also nonconvex. We obtain 50 demonstrations (Fig. \ref{fig:arm}.A, black) which minimize the joint-space path length, i.e. $c(\trajxu) = \sum_{t=1}^T \Vert q_{t+1} - q_t\Vert^2$ subject to a control constraint $\Vert q_{t+1} - q_t \Vert \le 0.1$, for all $t$. We train the GP for 150 epochs at learning rate 0.08, obtaining a GP whose posterior mean is visually consistent with the true constraint (Fig. \ref{fig:arm}.B). The posterior mean misclassifications are mostly on the interior of the C-space obstacle (as expected, since no data can be collected there), as well as minor errors on the constraint surface which are further away from the data. We plan via GP-CCRRT, taking two minutes on average, where 40 and 45 percent of the time is due to GP posterior and CDF computations, respectively. We use GP-CCRTT to obtain plans that are safe with probability greater than $0.9$; time-lapses of two such plans are shown in Fig. \ref{fig:arm}.D-E. Finally, we evaluate the baseline \cite{ral} with 10 boxes, timing out the optimization after 6 hours. Due to the number of demonstrations and constraint parameters, the baseline struggles (Fig. \ref{fig:arm}.F), returning boxes that do not adequately satisfy the KKT conditions and fail to capture the features of the true constraint. As before, the posterior mean has low ``false safe" and ``false unsafe" rates, and the ``false safe" rate can be reduced via buffering (see Table \ref{table:results}). In contrast, the baseline has a much higher ``false safe" rate, since it fails to cover most of the unsafe set, though it has a low ``false unsafe" rate, since it marks most of the space as safe. The NN baseline also fails to fit the gradient data, leading to low accuracy (Table \ref{table:results}). Overall, this example suggests that we can learn non-convex, non-workspace constraints on articulated robots while requiring minimal prior information, which is a necessity for C-space obstacles, which can be unintuitive.

\section{Discussion and Conclusion}

In this paper, we learn constraints from demonstrations with minimal \textit{a priori} knowledge by finding where the unknown constraint is tight and a scaling of its gradient at those points via the KKT conditions, and then training a GP-represented constraint that is consistent with and generalizes this data. We also show that the Gaussian structure of the GP uncertainty can exploited in an RRT-based planner to compute plans which satisfy the unknown constraint with a specified probability. Our results on a 5D car, 12D quadrotor, and 3-link planar arm show we can learn complex constraints on realistic systems which prior methods cannot handle. We conclude by discussing design choices and future work.

\noindent\textbf{Why use a GP constraint representation?} Our learning problem (fitting a function using zero level set data (the tight points) and its gradients at those points) closely relates to fitting manifolds \cite{ecomann} and signed distance functions \cite{signed_sdf} to data (though our method differs greatly in how it obtains the data, i.e. the KKT conditions). In both \cite{ecomann} and \cite{signed_sdf}, neural networks (NN) are used to fit large datasets on the order of $10^3 \sim 10^4$ and $10^5$ for \cite{ecomann} and \cite{signed_sdf}, respectively. In Sec. \ref{sec:results}, we show that in training the NN with only $\sim 10^2$ tight data points, the NN failed to provide an accurate fit. In contrast, derivative data can be directly embedded via the joint GP, which fits the constraint better with less data (on the order of $10^2$). Moreover, GPs handle constraint uncertainty in a principled way, which is crucial for safe planning.

\noindent\textbf{Limitations and future work}: First, as GPs are non-parametric, dense coverage of the constraint space with tight points is needed to reduce the predictive uncertainty. This results in cautious plans that stay near the training data, and may needlessly restrict the robot. In future work, we will explore semi-parametric models which combine a known, insufficient parameterization with a non-parametric GP to reduce uncertainty. Second, our method assumes demonstrations are precisely locally-optimal, but this is often untrue due to noisy observations or partial observability \cite{wafr20}; in future work, we will investigate suboptimality models (e.g. \cite{ziebart}) that can be used to adjust the confidence in extracted constraint value/gradient data. Third, we wish to extend our method to time-varying (e.g. temporal logic \cite{rss}) constraints. Finally, the scaling of GPs may hamper the learning of high-dimensional constraints; thus, we will explore scalable GP variants (e.g. sparse spectrum GP regression \cite{recht}). 

\bibliographystyle{IEEEtran}
\bibliography{condensed_ref}

\end{document}